\documentclass[letterpaper, 10 pt, conference]{ieeeconf}
\usepackage{amsmath,amsfonts}
\usepackage{algorithmic}
\usepackage{algorithm}
\usepackage{array}
\usepackage[caption=false,font=normalsize,labelfont=sf,textfont=sf]{subfig}
\usepackage{booktabs}
\usepackage{textcomp}
\usepackage{stfloats}
\usepackage{url}
\usepackage{verbatim}
\usepackage{graphicx}
\usepackage{cite}
\newcommand{\cH}{\mathcal{H}}
\newcommand{\MMD}{\text{MMD}}

\newtheorem{theorem}{Theorem}
\newtheorem{definition}{Definition}
\newtheorem{prop}{Proposition}
\newtheorem{lemma}{Lemma}

\begin{document}

\title{\LARGE \bf
Infinite-Horizon Ergodic Control via Kernel Mean Embeddings
}
\author{Christian Hughes and Ian Abraham}
\maketitle
\begin{abstract}
    This paper derives an infinite-horizon ergodic controller based on kernel mean embeddings for long-duration coverage tasks on general domains. 
    While existing kernel-based ergodic control methods provide strong coverage guarantees on general coverage domains, their practical use has been limited to sub-ergodic, finite-time horizons due to intractable computational scaling, prohibiting its use for long-duration coverage. 
    We resolve this scaling by deriving an infinite-horizon ergodic controller equipped with an extended kernel mean embedding error visitation state that recursively records state visitation. 
    This extended state decouples past visitation from future control synthesis and expands ergodic control to infinite-time settings. 
    In addition, we present a variation of the controller that operates on a receding-horizon control formulation with the extended error state. 
    We demonstrate theoretical proof of asymptotic convergence of the derived controller and show preservation of ergodic coverage guarantees for a class of 2D and 3D coverage problems. 
\end{abstract}


\section{Introduction}

    Modern autonomous systems are tasked with long-duration coverage and exploration in remote and complex environments. 
    In problem domains like underwater monitoring or planetary exploration, agents must distribute their time across environments in proportion to the utility, or importance, of each area. 
    Proportional coverage ensures that high-priority areas receive frequent attention while maintaining complete visitation of the entire coverage domain.
    
    Ergodic control methods provide a robust approach to proportional coverage tasks by minimizing the discrepancy between how often an agent visits an area (a time-average distribution) and a target spatial distribution \cite{mathew2011}. 
    Classical ergodic control methods optimize spectral ergodic metrics to produce proportional coverage trajectories in continuous-space environments, but these methods are generally restricted to bounded Euclidean domains using well-defined Fourier basis functions \cite{lee2024, sartoretti2021}. 
    Recent work has shown that kernel-based metrics like maximum mean discrepancy (MMD) successfully metrize weak convergence between trajectory and target measures \cite{simon2023} and enable ergodic path planning on non-Euclidean manifolds via a Reproducing Kernel Hilbert Space \cite{hughes2025, sun2024fast}.

    Despite their generality to arbitrary domains, kernel-based ergodic metrics suffer from a fundamental computational bottleneck over long-time horizons. 
    In order to ensure ergodicity in long time-horizons, the control objective must account for the agent's entire visitation history. 
    Visitation recording is commonly treated as a cumulative sampling problem where history is represented as a growing list of discrete state observations \cite{miller2016} while others formulate a recursive approximation of the spectral ergodic metric~\cite{mavrommati2018real, de_la_torre_2016} that facilitates infinite-horizon control. 
    The goal in this work is to derive an equivalent formulation for the kernel mean embedding formulation of ergodic control~\cite{hughes2025, sun2024fast} for the infinite-horizon setting.

    \begin{figure}[t!]
        \centering
        \vspace{5pt} 
        \includegraphics[width=\linewidth]{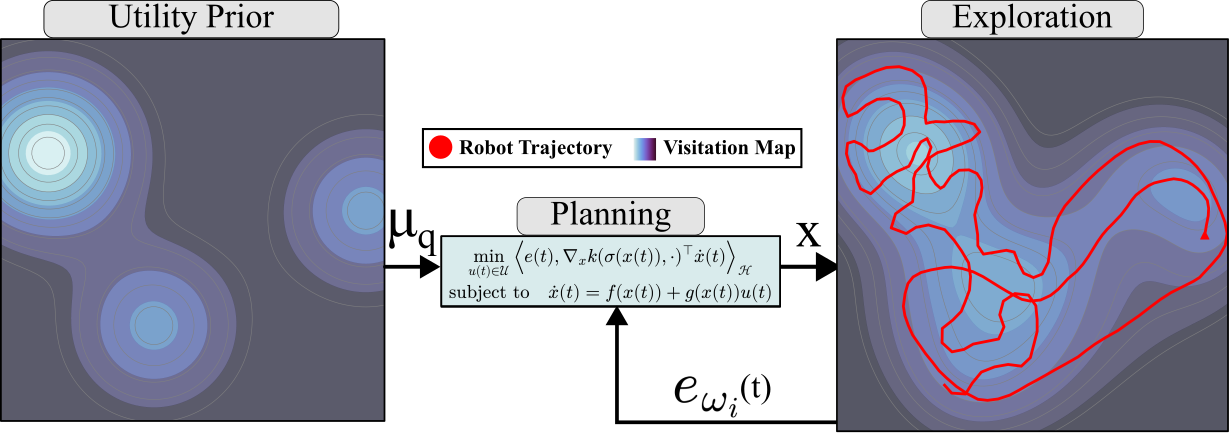}
        \vspace{-20pt}
        \caption{\textbf{Infinite-Horizon Ergodic Control Feedback Loop.} The desired target distribution is embedded into a kernel mean embedding function $\mu_q$ that is tracked. At each time-step, the agent computes an ergodic control based on its current state $x(t)$ and its visitation embedding $\mu_{\rho_{x,t}}$ according to its time-average visitation distribution $\rho_{x,t}$.
        The ergodic controller then seeks to minimize the distance between these embeddings $e_{\omega_i}(t)$ at distinct points in the domain $\omega_i$ as $t\to\infty$, inducing ergodicity.}
        \label{fig:kme_diagram}
        \vspace{-10pt}
    \end{figure}

    To achieve our goal, we derive an infinite-horizon ergodic controller based on kernel mean embeddings (KME) to represent visitation history. 
    We show that the derived infinite-horizon ergodic controller induces an error visitation state that records past visitations in a single embedding which decouples past states from future control synthesis.
    In addition, we introduce a recursive model-predictive control variation that produces infinite-horizon ergodic controls that converge more quickly towards ergodicity. 
    This approach ensures the controller can asymptotically achieve ergodic coverage (and thus ergodicity) over long mission durations.
    We demonstrate our approach on a number of coverage problems in 2D and 3D, and provide analysis on the impact of receding-horizon control as well as kernel selection.  
    In summary, the primary contributions of this work are:
    \begin{itemize}
        \item Derivation of an infinite-horizon ergodic controller based on kernel mean embeddings.
        \item Introduction of an extended error visitation state that tracks past visitations in a compact form.
        \item Presentation of a receding-horizon formulation with the extended error visitation state. 
        \item Theoretical analysis of asymptotic convergence of the infinite-horizon ergodic controller. 
        \item Experimental validation on 2D and 3D manifolds demonstrating effective coverage over long-durations. 
    \end{itemize}
    
    The remainder of this paper is organized as follows: Section~\ref{sec:preliminaries} provides preliminaries on ergodicity and MMD; Section \ref{sec:main} derives the proposed recursive framework; and Sections \ref{sec:results} and \ref{sec:conclusion} present experimental results and concluding remarks.

\section{Preliminaries}
\label{sec:preliminaries}

In this section, we introduce the problem of coverage control from the perspective of satisfying ergodicity as a controllable property. An introduction to kernel mean embeddings and its connection to ergodicity and ergodic control is provided. 

\subsection{Ergodicity}

    We consider the class of deterministic control-affine systems of the form
    \begin{equation}\label{eq:dynamics}
        \dot{x}=f(x)+g(x)u,
    \end{equation}
    where $x \in \mathcal{X} \subset \mathbb{R}^n$ is the state, $u \in \mathcal{U}\subset \mathbb{R}^m$ is the control input, $f:\mathcal{X}\to \mathbb{R}^n$ is the drift vector field, and $g:\mathcal{X}\to \mathbb{R}^{n\times m}$ is the input matrix field.
    Next, let us define a point on the coverage manifold and an twice differentiable map $\sigma : \mathcal{X} \to \Omega$ that converts state to a visited point on $\Omega$. 

    \begin{definition} \textit{(Time-Averaged Trajectory Distribution)}
        Let $x(t) : \mathbb{R}_+ \to \mathcal{X}$ be a trajectory following~\eqref{eq:dynamics} subject to a controller $u(t) : \mathbb{R}_+ \to \mathcal{U}$ at some initial condition $x_0 = x(0)$. The time-averaged trajectory distribution on $\Omega$ is defined as 
        \begin{equation}
            \rho_{x,t}(\omega) = \frac{1}{t} \int_0^t \delta[\omega- \sigma(x(\tau))] d\tau,
            \label{eq:time_averaged_distribution}
        \end{equation}   
        where $\delta$ is the Dirac-delta function.
    \end{definition}

    \begin{definition} \textit{(Ergodicity)} \label{def:ergodicity}
        ~\cite[Theorem 6.14]{walters_introduction_2000} Let $q : \Omega \to \mathbb{R}_+$ be a probability measure on $\Omega$, i.e., $q \in \mathcal{P}(\Omega)$. A trajectory $x(t)$ is said to be ergodic with respect to $q$ if its time-averaged trajectory distribution $\rho_{x,t}$ at the limit $t\to \infty$ converges weakly to $q$. More concretely, 
        \begin{equation*}
            \lim_{t\to \infty} \int_\Omega \phi(\omega) d \rho_{x, t}(\omega) = \int_\Omega \phi(\omega) dq(\omega)
        \end{equation*}
        for all continuous functions $\phi \in \mathcal{C}(\Omega)$, where,
        \begin{equation*}
            \int_\Omega \phi(\omega) d \rho_{x, t}(\omega) = \frac{1}{t}\int_{0}^{t} \phi(\sigma(x(\tau))) d\tau,
        \end{equation*}
        making the equality
        \begin{equation} \label{eq:ergodicity_limit}
            \lim_{t\to \infty} \frac{1}{t}\int_{0}^{t} \phi(\sigma(x(\tau)))d\tau = \int_\Omega \phi(\omega) dq(\omega).
        \end{equation}
    \end{definition}

    Intuitively, a trajectory that satisfies~\eqref{eq:ergodicity_limit} traverses all points in $\Omega$ proportional to the expected value of $q$ over $\Omega$ which is ideal for coverage and exploration problems. However, many systems are not naturally ergodic and require some form of controller to induce ergodicity. 

\subsection{Metrics on Ergodicity}

    Ergodicity can be rigorously quantified by defining a metric that evaluates the difference between a trajectory's time-averaged distribution and a given target distribution $q$. 
    Specifically, the distance to ergodicity~\eqref{def:ergodicity} given a trajectory $x(t)$ and target distribution $q$ as
    \begin{equation} \label{eq:generic_ergodic_metric}
        \mathcal{E}(x, q, t) = \Big\Vert \mathcal{F}[\rho_{x,t}] - \mathcal{F}[q] \Big\Vert_\mathcal{H}^2
    \end{equation}
    where $\mathcal{F}: \mathcal{P}(\Omega) \to \mathcal{H}$ is a integral transform~\cite{mathew2011, hughes2025} and $\Vert \cdot \Vert_\mathcal{H}$ is the associated functional norm.
    Ergodicity is then equivalently satisfied when the following holds 
    \begin{equation*}
        \lim_{t\to\infty} \mathcal{E}(x, q, t) = 0.
    \end{equation*}
    One can then formulate the optimal sub-ergodic control problem as 
    \begin{align} \label{eq:ergodic_optimal_control}
        \min_{u(t) \forall t \in [0, t_f]}\,\, & \mathcal{E}(x,q, t_f)  + \int_{0}^{t_f} \ell(x(t), u(t)) dt\\         
        \text{subject to } & \dot{x} = f(x) + g(x) u, \nonumber \\
        & u(t)\in\mathcal{U} \,\, \forall t \in [0, t_f] \nonumber \\ 
        \text{given }& x_0=x(0), q, \sigma \nonumber
    \end{align}    
    over the more tractable finite-time horizon $t_f<\infty$ where $\ell : \mathcal{X} \times \mathcal{U} \to \mathbb{R}$ is an additional running cost. \footnote{The resulting controller yields sub-ergodic trajectories rather than optimal ergodic trajectories.}
    The main challenge is that the finite-time restriction is necessary for solving ergodic controllers as infinite-time formulations are computationally challenging to optimize. 
    
    Prior work has solved these ergodic trajectory optimization problems in finite-horizon using the Fourier transform $\mathcal{F}$ \cite{miller2016, lee2024} where more recent approaches adopt general kernel-based techniques to approximate the ergodic metric~\cite{sun2024fast, hughes2025} over finite-time horizons. 
    In general, the infinite-time ergodic control problem has only been solved through manipulation of the Fourier ergodic metric~\cite{mathew2011, de_la_torre_2016} that is restricted to domains where the Fourier transform is well-defined, or through a heat-equation variation of ergodic control that requires solving a costly partial-differential equation at all times~\cite{ivic2022, bilaloglu2024tactile}. 
    In this work, we seek to derive the infinite-time ergodic controller through the more general, kernel formulation that allows more general specification of $\Omega$ as shown in~\cite{hughes2025}.

\subsection{Kernel Mean Embedding}

    An alternate formulation of an ergodic metric in~\eqref{eq:generic_ergodic_metric} can be derived through kernel mean embeddings as stand-in for $\mathcal{F}$ in a Reproducing Kernel Hilbert Space (RKHS).
    Let $k: \Omega \times \Omega \to \mathbb{R}_{\ge 0}$ be a positive-definite kernel on a compact domain $\Omega$, with $\cH$ as the associated RKHS. 
    The kernel mean embedding for a probability distribution $p$ is given by
    \begin{align}
        p \mapsto \mu_p = \mathbb{E}_{\omega \sim p}[ k(\omega, \cdot)].
    \end{align}
    A kernel is \emph{characteristic}~\cite{sriperumbudur2011} if this mapping is injective and a unique representation of $p$ in $\cH$.
    Given two probability distributions $p$ and $q$ on $\Omega$, the maximum mean discrepancy (MMD)~\cite{gretton2012} measures the difference between the distributions by comparing their kernel mean embeddings in the RKHS, 
    \begin{align}\label{eq:mmd_definition}
        \MMD_k^2(p, q) = \|\mu_p - \mu_q\|_{\cH}^2.
    \end{align}
    \begin{theorem}[Weak Convergence of MMD, \cite{simon2023}, Theorem 7]
        \label{theorem:weak_convergence}
        A bounded and measurable kernel metricizes the weak convergence of probability measures on a compact Hausdorff space $\Omega$ if and only if that kernel is both characteristic and continuous.
    \end{theorem}
    Theorem~\ref{theorem:weak_convergence} supports that MMD is a valid metric for ergodicity. In addition, it is valuable to define the expanded variation of the MMD metric. 
    \begin{lemma}[Expanded MMD, \cite{gretton2012}, Lemma 4]
    Let $x, x'$, and $y, y'$ be independent variables with distribution $p, q$ respectively on $\Omega$. 
    MMD can be expanded to obtain the following
    \begin{align} \label{eq:expanded_mmd}
            \MMD_k^2(p, q) =& \mathbb{E}_{x,x'}[k(x, x')] \\ &- 2\mathbb{E}_{x,y}[k(x, y)] + \mathbb{E}_{y, y'}[k(y, y')], \nonumber
    \end{align} 
    where $x' \sim p$ and $y' \sim q$ are independent copies of $x$, $y$ with the same distribution. 
    \end{lemma}
    The advantage of MMD is that one does not need access to the underlying distributions to compute the metric through Monte-Carlo sampling of the expected values. 

\subsection{Ergodic MMD Metric}

    Given a trajectory $x(t)$ and its time-averaged distribution $\rho_{x,t}$, the ergodic MMD metric can readily be formulated.

    \begin{prop}[Ergodic MMD Metric]
        Given a target distribution $q : \Omega \to \mathbb{R}_+$, trajectory $x(t) : \mathbb{R}_+ \to \mathcal{X}$ and its time-averaged distribution $\rho_{x,t}$ the ergodic maximum mean discrepancy metric for $t>0$ is given as 
        \begin{align}
            \mathcal{E}_{\text{MMD}}(x, q, t) &=
            \frac{1}{t^2}\iint_{t,t} k(\omega_{x,\tau}, \omega_{x,\tau'}) d\tau d \tau'
            \\ & \quad - \frac{2}{t} \iint_{\Omega, t} k(\omega_{x,
        \tau}, \omega) d\tau dq + z(q), \nonumber
    \end{align}
    where $\omega_{x,\tau} = \sigma(x(\tau))$ and $z(q) = \iint_{\Omega, \Omega} k(\omega, \omega') dq dq$ is a normalization constant.
    \end{prop}
    \begin{proof}
        Substituting $\rho_{x,t}$ and $q$ into the definition of MMD in~\eqref{eq:mmd_definition} and then applying the expanded MMD formulation in~\eqref{eq:expanded_mmd} we get the following
        \begin{align}
            \text{MMD}_k^2(\rho_{x,t}, q) &= \|\mu_{\rho_{x,t}} - \mu_q\|_{\cH}^2 \\
            &= \mathbb{E}_{\omega, \omega' \sim \rho_{x,t}}[k(\omega, \omega')] \label{eq:full_exp}\\ 
            &\hspace{2em} - 2\mathbb{E}_{\omega \sim \rho_{x,t}, \omega' \sim q}[k(\omega, \omega')] \nonumber \label{eq:reduced_exp}\\
            &\hspace{4em}+ \mathbb{E}_{\omega, \omega' \sim q}[k(\omega, \omega')] \nonumber
             \\ &=
            \frac{1}{t^2}\iint_{t,t} k(\omega_{x,\tau}, \omega_{x,\tau'}) d\tau d \tau
            \\ & \hspace{2em} - \frac{2}{t} \iint_{\Omega, t} k(\omega_{x,\tau}, \omega) d\tau dq, \nonumber
             \\ & \hspace{4em} + \iint_{\Omega, \Omega} k(\omega, \omega') dq dq \nonumber \\ 
                \mathcal{E}_\text{MMD}(x, q, t) &= \frac{1}{t^2}\iint_{t,t} k(\omega_{x,\tau}, \omega_{x,\tau'}) d\tau d \tau
             \nonumber \\ 
             & \hspace{2em} - \frac{2}{t} \iint_{\Omega, t} k(\omega_{x,\tau}, \omega) d\tau dq + z(q), \nonumber 
        \end{align}
        where~\eqref{eq:reduced_exp} is the result of the Dirac-delta in $\rho_{x,t}$ that reduces the expected values to~\eqref{eq:full_exp} and $z(q)$ is the normalization constant that does not depend on $x$.
    \end{proof}

    Computing the integral over $\Omega$ can be computationally expensive in practice. 
    Instead, it is common to approximate the integral through Monte-Carlo samples $\{\omega_i \}_{i=1}^M \sim q$ as 
    \begin{align} \label{eq:ergodic_mmd_approx}
            \mathcal{E}_{\text{MMD}}(x, q, t) \approx&
            \frac{1}{t^2}\iint_{t,t} k(\omega_{x,\tau}, \omega_{x,\tau'}) d\tau d \tau'
            \\ & \quad - \frac{2}{t M} \sum_{i=1}^M \int_{t} k(\omega_{x,t}, \omega_i) d\tau \nonumber \\ 
            & \quad \quad \quad \quad+ \frac{1}{M^2} \sum_{i=1}^M \sum_{j=1}^M k(\omega_i, \omega_j). \nonumber
    \end{align}
    It is possible to now use~\eqref{eq:ergodic_mmd_approx} in a optimal control problem as described in~\eqref{eq:ergodic_optimal_control} for a finite time $t>0$.

\section{Infinite-Horizon Ergodic Control}
\label{sec:main}

    In this section, we derive the infinite-horizon ergodic controller based on kernel mean embeddings and then prove the convergence of the proposed control law to the target distribution.

    \subsection{Time-Aug. Ergodic MMD with Visitation Error States}
    In order to derive the infinite time ergodic controller, we first adapt the metric $\mathcal{E}_{\text{MMD}}$ with a temporal scaling factor $t^2$.
    \begin{definition}
        \textit{(Time-Augmented Ergodic MMD Metric)} The time-augmented ergodic MMD metric is defined as $\tilde{\mathcal{E}}(x, q, t) = |t|^2\mathcal{E}_{\text{MMD}}(x, q, t)$, or equivalently
        \begin{align}
            \tilde{\mathcal{E}}(x, q, t) &= |t|^2\mathcal{E}_{\text{MMD}}(x, q, t) \\ 
            &= |t|^2 \|\mu_{\rho_{x, t}} - \mu_q \|_{\mathcal{H}}^2 \nonumber 
        \end{align}
    \end{definition}
    \begin{prop}[Visitation Error State Representation of Ergodic MMD Metric]
        Let $e(x,q, t) = t \left(\mu_{\rho_{x, t}} - \mu_q\right)$ be known as the visitation error state. 
        Then, the time-augmented ergodic MMD metric $\tilde{\mathcal{E}}_{\text{MMD}}$ is equivalently represented as 
        \begin{align}
            \tilde{\mathcal{E}}_{\text{MMD}}(x, q, t) = \|e(x,q, t)\|_{\mathcal{H}}^2
        \end{align}
        where the visitation error state $e(x,q, t)$ is the solution to the following initial value problem:
        \begin{align} \label{eq:ivp_error_state}
            \frac{d}{dt} e(x,q, t) &= k(\sigma(x(t)), \cdot) - \mu_q, \quad e(x,q, 0) = 0. \\ 
            \frac{d}{dt} x &= f(x) + g(x) u, \quad x(0) = x_0. \nonumber
        \end{align}
    \end{prop}
    \begin{proof}
        It is straightforward to verify that $|t|^2 \mathcal{E}_{\text{MMD}}(x, q, t) = \|e(x,q, t)\|_{\mathcal{H}}^2 = \|t (\mu_{\rho_{x, t}} - \mu_q)\|_{\mathcal{H}}^2 = |t|^2 \|\mu_{\rho_{x, t}} - \mu_q\|_{\mathcal{H}}^2$.
        To obtain the initial value problem, we expand the expression $e(x,q, t) = t (\mu_{\rho_{x, t}} - \mu_q)$ to yield the following integral form
        \begin{align*}
            e(x,q, t) &= \int_\Omega k(\omega, \cdot) \left( \int_0^t \delta[\omega - \sigma(x(\tau))] d\tau \right) d\omega -t\mu_q\\
            &= \int_0^t k(\sigma(x(\tau)), \cdot) d\tau - t \mu_q \\
            &= \int_0^t k(\sigma(x(\tau)), \cdot)  -  \mu_q \, d\tau.
        \end{align*}
        Taking the time derivative and applying the Leibniz integral rule yields
        \begin{align} \label{eq:error_update}
            \frac{d}{dt} e(x,q, t) &= \frac{d}{dt} \left( \int_0^t k(\sigma(x(\tau)), \cdot) - \mu_q \, d\tau \right) \nonumber \\
            &= k(\sigma(x(t)), \cdot) - \mu_q
        \end{align}     
        with the initial condition $e(x,q, 0) = 0$.   
    \end{proof}
    Note that the visitation error state $e(x,q, t)$ is a functional over $\Omega$ that is integrated out through the RKHS $\mathcal{H}$ norm.
    We define $e(t) = e(x, q, t)$ as the visitation error state functional along the trajectory $x(t)$, that is, $e(t)$ tracks the deviation of the trajectory from the target distribution over time. 
    In practice, we represent $e(t)$ as a vector of evaluations on a set of sampled points $\{\omega_i\}_1^{M} \sim q$, where $e_{\omega_i}(t)$ is the evaluation of the error state at point $\omega_i$.

\subsection{Infinite-Horizon Ergodic Control via KME}
    Given the visitation error state representation of the ergodic MMD metric, we can derive a steepest descent control law to drive the error state toward zero as $t\to \infty$.
    \begin{theorem}[Infinite-Horizon Ergodic Control] \label{thm:convergence}
        Consider the initial value problem defined in~\eqref{eq:ivp_error_state} for a given target distribution $q \in \mathcal{P}(\Omega)$, RKHS $\mathcal{H}$, and characteristic kernel $k$.
        Assuming smoothness and boundedness conditions on $f$, $g$, and $k$, 
        the steepest descent control law that locally minimizes the ergodic MMD metric is given by
        \begin{align}\label{eq:ergodic_ctrl_law}
            u(x, e, t) = - \alpha \left( g(x)^\top \int_\Omega e_\omega(t) \nabla_x k(\sigma(x(t)), \omega) \, d\omega \right)
        \end{align}
        where $\alpha : \mathbb{R}^m \to \mathcal{U}$ is a projection operator that ensures the control input remains within the admissible set $\mathcal{U}$.
    \end{theorem}
    \begin{proof}
        We first compute the second-order Taylor expansion of $\tilde{\mathcal{E}}_\text{MMD}$ centered around $t$ as follows
        \begin{align*}
            \tilde{\mathcal{E}}_\text{MMD}(x, q, t+\Delta t) &\approx \tilde{\mathcal{E}}_\text{MMD}(x, q, t) + 2\Delta t \langle e(t), \dot{e}(t) \rangle_{\mathcal{H}} \\
            & \quad + (\Delta t)^2 \left( \langle \dot{e}(t), \dot{e}(t)\rangle_{\mathcal{H}} + \langle e(t), \ddot{e}(t) \rangle_{\mathcal{H}} \right).
        \end{align*}
        Using the expression of $\dot{e}(t)$ from~\eqref{eq:ivp_error_state}, and noting that 
        \begin{align*}
            \ddot{e}(t) & = \nabla_x k(\sigma(x(t)), \cdot)^\top \dot{x}(t) \\
            & = \nabla_x k(\sigma(x(t)), \cdot)^\top (f(x(t)) + g(x(t)) u(t))
        \end{align*}
        is the only term that directly depends on the control input $u(t)$, we formulate the following steepest descent optimization 
        \begin{align*}
            \min_{u(t) \in \mathcal{U}} \Big \langle e(t), \nabla_x k(\sigma(x(t)), \cdot)^\top \dot{x}(t) \Big \rangle_{\mathcal{H}} \\ 
            \text{subject to} \quad \dot{x}(t) = f(x(t)) + g(x(t)) u(t).
        \end{align*}
        For $\mathcal{U} = \{ u \in \mathbb{R}^m \mid \|u\| \leq u_\text{max} \}$, the optimization problem has the exact solution 
        \begin{equation*}
            u(x, e, t) = - \alpha \left( g(x)^\top \int_\Omega e_\omega(t) \nabla_x k(\sigma(x(t)), \omega) \, d\omega \right)
        \end{equation*}
        where $\alpha : \mathbb{R}^m \to \mathcal{U}$ is a projection operator that ensures the control input remains within the admissible set $\mathcal{U}$.
    \end{proof}
    
    In the scenario where $g(x)$ is not full rank, the control law can be modified to project the descent direction onto the range of $g(x)$.
    Additionally, the control law can be implemented in a receding horizon fashion by solving the optimization problem over a finite time horizon $t_h$ and applying the control input for a short duration before re-solving the optimization problem at the next time step which we demonstrate in the following sections.
    In summary, the optimization formulation can be solved numerically using off-the-shelf solvers and is given as 
    \begin{align}\label{eq:plannig_opt}
        &\min_{u(t) \in [t,t+t_h]} \tilde{\mathcal{E}}(x, q, t+t_h) \approx \frac{1}{M}\sum_{i=1}^M e_{\omega_i}(x, q, t+t_h)^2 \nonumber \\
        &\text{subject to } \text{Eq.}~\eqref{eq:ivp_error_state} \text{ given } x(t), e(x,q, t)
    \end{align}
    where $x(t), e(x, q, t)$ is the system state at time $t$.
    
    The primary advantage of the proposed visitation error state representation is that it acts like a fixed memory of where the system has visited in the past. As a result, the ergodic system becomes independent from the total mission duration. 
    In standard MMD-based ergodic control \eqref{eq:ergodic_mmd_approx}, the explicit integration of time required tracking some visitation history to optimize for control.  
    This led to scaling issues with the time-horizon, namely $\mathcal{O}((t_h + t_p)^2)$, where $t_p$ is the past trajectory time and $t_h$ is the planning horizon. 
    For long-duration missions, the computational cost becomes prohibitive for real-time control. 
    In contrast, by representing visitation as an embedding via $e$, the proposed control law \eqref{eq:ergodic_ctrl_law} scales as $\mathcal{O}(M)$ and $\mathcal{O}(t_h + M)$, if using a discrete integration $N=t_h/\Delta t$ in~\eqref{eq:plannig_opt}.

\begin{figure*}
    \centering
    \includegraphics[width=\linewidth]{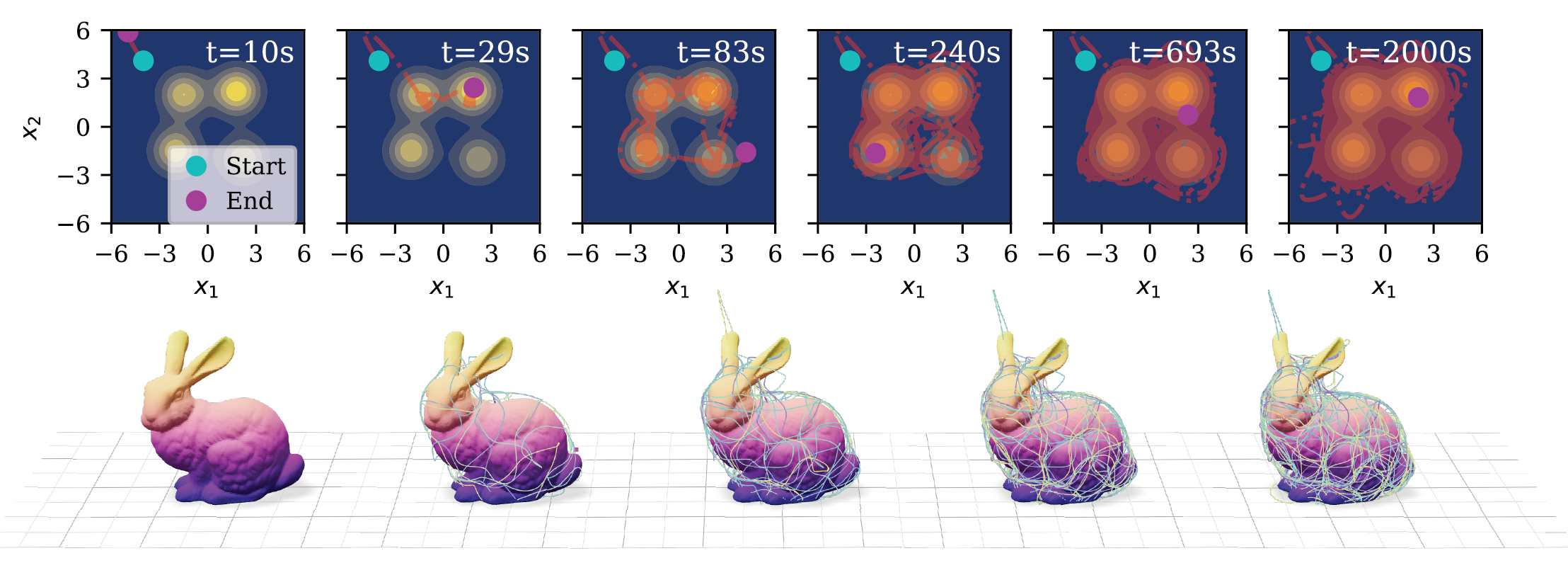}
    \vspace{-20pt}
    \caption{\textbf{Time Series Evolution of Ergodic Controller~\eqref{eq:ergodic_ctrl_law}.} Illustrated are time snapshots of the infinite-time ergodic controller derived form the kernel mean embedding for (top) a planar multi-modal target distribution and (bottom) a bunny manifold with uniform distribution. The dynamics are given by the single-integrator system and the error visitation state is tracked via~\eqref{eq:ivp_error_state} for $M=500$ samples. As $t\to \infty$, the ergodic controller fully covers the domain without increasing computation. }
    \label{fig:time_series}
    \vspace{-15pt}
\end{figure*}
    
\subsection{Asymptotics}

    We can show that under certain regularity conditions, the steepest descent control law in \eqref{eq:ergodic_ctrl_law} ensures that the ergodic MMD metric converges to zero as $t \to \infty$.
    \begin{theorem}[Ergodicity of Closed-Loop Ergodic Control]
        Consider the single-integrator dynamics $\dot{x} = g(x)u$ where $\mathcal{X},  \mathcal{U}$ are compact, $g(x)$ is full rank, and for simplicity $\alpha$ is the identity map.
        Then, the steepest descent control law in \eqref{eq:ergodic_ctrl_law} satisfies the following closed-loop dynamics
        \begin{equation}            
            \dot{x} = - g(x)^\top \int_\Omega e_\omega(t) \nabla_x k(\sigma(x(t)), \omega) \, d\omega,
        \end{equation}
        and generates a trajectory that asymptotically converges to the target distribution $q$ in the ergodic sense that $\lim_{t \to \infty} \tilde{\mathcal{E}}_{\text{MMD}}(x, q, t) = 0$.
    \end{theorem}
    \begin{proof}
        Substituting the control law into the closed-loop dynamics yields $\dot{x} = - g(x)^\top \int_\Omega e_\omega(t) \nabla_x k(\sigma(x(t)), \omega) \, d\omega$.
        Defining $\tilde{\mathcal{E}}(x,q,t+\Delta t) - \tilde{\mathcal{E}}(x, q, t)-2\Delta t \langle e(t), \dot{e}(t) \rangle_\mathcal{H} - \Delta t^2 \| \dot{e}(t) \|_\mathcal{H}^2 = \Delta \tilde{\mathcal{E}}$ for simplicity and substituting the control law into the second-order Taylor expansion of $\tilde{\mathcal{E}}$ yields
        \begin{align} \label{eq:second_order_expansion_closed_loop}
            \Delta \tilde{\mathcal{E}} & \approx \langle e(t), \ddot{e}(t) \rangle_{\mathcal{H}} \\
            & = \langle e(t), \nabla_x k(\sigma(x(t)), \cdot)^\top g(x) u(x,e,t) \rangle_{\mathcal{H}} \nonumber \\
            & = - \int_\Omega \int_\Omega \| g(x)^\top e_\omega(t) \nabla_x k(\sigma(x(t)), \omega)  \|^2  d\omega d\omega' \nonumber \\
            & \leq 0. \nonumber
        \end{align}
        Given $\mathcal{X}$ is compact and $k$ is smooth, $\dot{x}$, $e$, $\dot{e}$, are bounded and $\ddot{\tilde{\mathcal{E}}}_\text{MMD}$ is bounded and $\dot{\tilde{\mathcal{E}}}_\text{MMD}$ is uniformly continuous by inspection. By Barbalat's Lemma \cite{barbalat1959} and that~\eqref{eq:second_order_expansion_closed_loop} approximates the change in the ergodic MMD metric, $\lim_{t \to \infty} \dot{\tilde{\mathcal{E}}}_\text{MMD}(t) = 0$, which necessitates $\lim_{t \to \infty} \tilde{\mathcal{E}}_\text{MMD}(t) = 0$.
    \end{proof}
    The above result shows that the proposed steepest descent control law ensures that the trajectory distribution converges to the target distribution in the ergodic sense as time goes to infinity.
    In practice, the integral over $\Omega$ can be approximated through Monte-Carlo sampling as described in~\eqref{eq:ergodic_mmd_approx} and the control law can be implemented in a receding horizon fashion to achieve ergodic coverage over a finite time horizon.

\section{Results}
\label{sec:results}

\subsection{Robustness of Ergodic Convergence to Planning Horizon}

\begin{figure}[t!]
    \centering
        \includegraphics[clip, trim={0.5em 0.5em 0em 1.8em}, width=\linewidth]{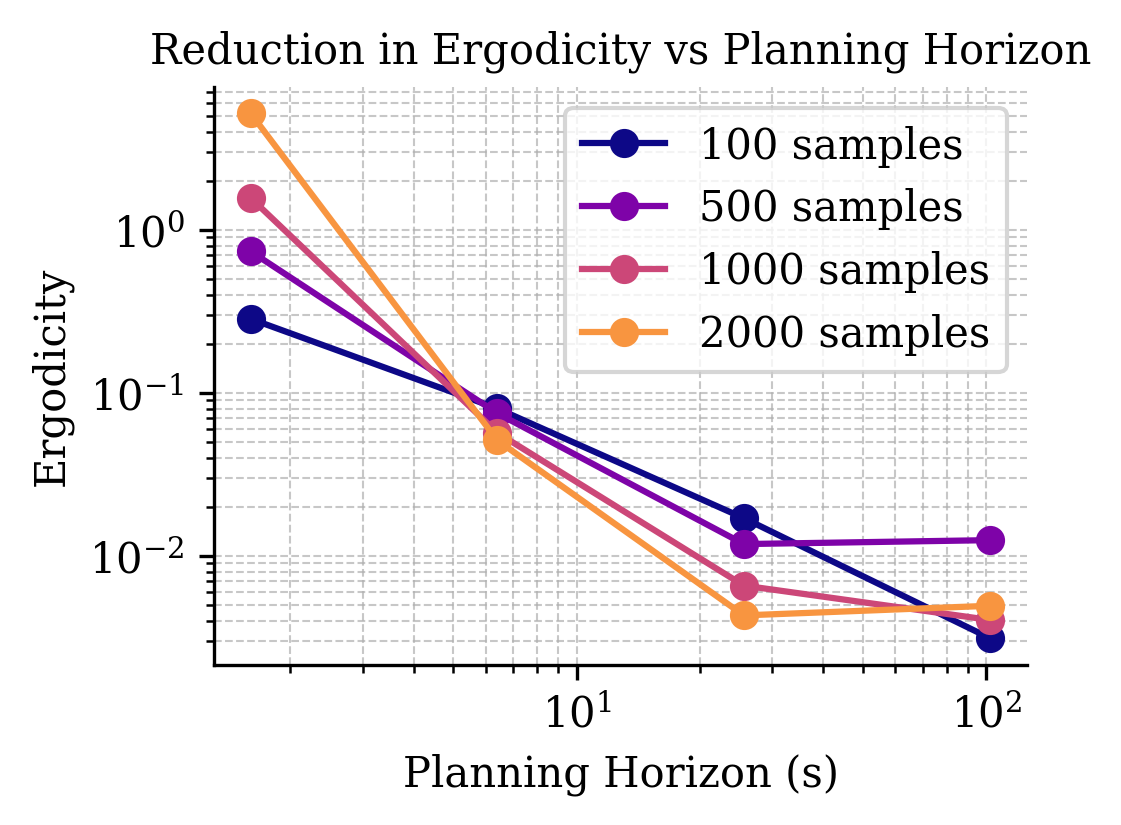}
    \vspace{-20pt}
    \caption{\textbf{Reduction of Ergodicity vs. Planning Horizon}. Here, we illustrate the impact that the receding-horizon formulation in~\eqref{eq:plannig_opt} has on the reduction of ergodicity over the planning horizon for a single-integration dynamical system. As the planning horizon increases, so does the reduction in ergodicity. Note that for large samples, numerical precision becomes a limited factor which motivates the use of the extended error visitation state which is a direct tradeoff between planning controls versus reactive control. }
    \vspace{-15pt}
    \label{fig:convergence}
\end{figure}

\begin{figure}[t!]
    \centering
    \vspace{5pt} 
    \includegraphics[width=\linewidth]{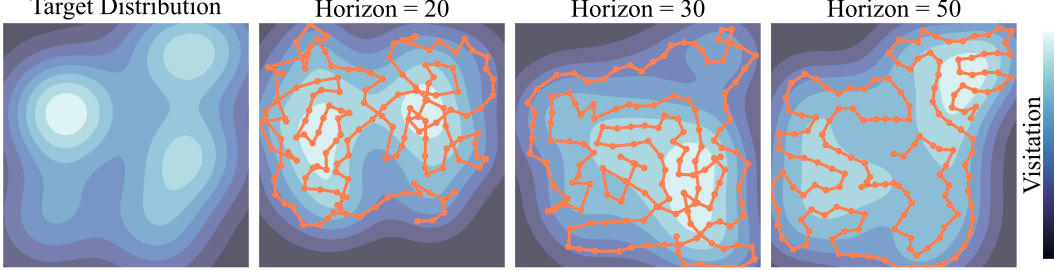}
    \vspace{-15pt}
    \caption{\textbf{Impact of MPC Planning Horizon on Short-Term Ergodic Convergence.} 
    Here, we show a qualitative illustration of the impact of control planning horizon on the effective coverage trajectory. As the time-horizon increases, the trajectory tries to span more of domain proportional to the target distribution. However, as planning time becomes longer, so does the complexity of the trajectory and control, leading to numerical instabilities. 
    }
    \label{fig:hor_ablation}
    \vspace{-10pt}
\end{figure}

\subsubsection{Problem Formulation}
Consider the receding-horizon optimization problem defined in \eqref{eq:plannig_opt} over a finite horizon $N \in \mathbb{Z}_+$.
We aim to characterize the sensitivity of the ergodic metric $\mathcal{E}(t)$ to variations in $N$ for a single-integrator system operating on a compact domain $\Omega \in \mathbb{R}^n$. 
The objective is to verify that the control law \eqref{eq:ergodic_ctrl_law} preserves asymptotic convergence to the target distribution $q$ in arbitrary planning horizons. 
We define the mission duration as $T>0$ and evaluate the resultant approximated visitation distribution $\mu_{\rho_{x,t}}$ against the target $\mu_q$.

\subsubsection{Results}
The influence of the planning horizon $N$ on the convergence rate of ergodicity is shown qualitatively in Fig. \ref{fig:hor_ablation} and quantitatively in Fig. \ref{fig:convergence}. 
We prescribe a target distribution on a bounded domain $\Omega = [-1, 1]^2$ and evaluate trajectories for $N\in\{30, 60, 100\}$ with a fixed mission length $T=150$.

As $N$ increases, so too does the optimizer's ability to account for global information density, which results in a trajectory that more efficiently balances the repulsion from the current state $x(t)$ with the attraction to high-probability regions in $\mu_q$. 
This behavior is consistent with the performance characteristics of traditional model predictive control in constrained regulation tasks. 

To establish a lower bound of robustness, we analyze the degenerate case where $N=1$. 
Figures \ref{fig:time_series} \& \ref{fig:one_horizon} demonstrate that even with minimal lookahead, the control law in \eqref{eq:ergodic_ctrl_law} successfully drives the agent's visitation towards proportionality to $\mu_q$. 
These results confirm that the proposed ergodic control law does not rely on large planning horizons to guarantee ergodicity, as the visitation embedding within \eqref{eq:ergodic_ctrl_law} ensures that the agent's future visitation remains proportionally consistent with $\mu_q$ for all $N\ge1$.



\begin{figure}[h]
    \centering
        \includegraphics[clip, trim={1em 0.5em 0 0}, width=\linewidth]{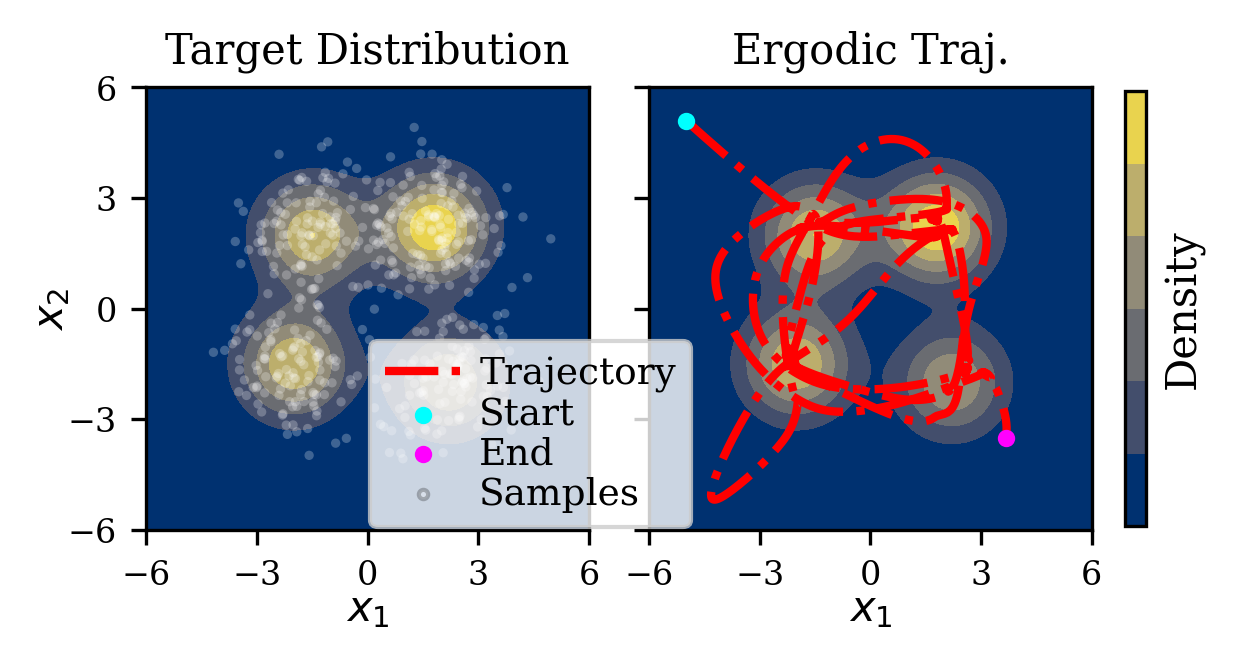}
    \vspace{-18pt}
    \caption{\textbf{Illustration of Infinite-Horizon Feedback Control.} (Left) A target reference distribution $q$ on $\Omega$ with samples $\{ \omega_i \}_{i=0}^M \sim q$. (Right) An ergodic trajectory from a single-integrator dynamical system following the proposed control law~\eqref{eq:ergodic_ctrl_law}.
    }
    \label{fig:one_horizon}
\end{figure}

\subsection{Parameter Sensitivity \& Physical Consistency}

\subsubsection{Problem Formulation}
The performance of the proposed ergodic controller is intrinsically coupled to the spectral properties of the visitation embedding's kernel function $k(\cdot, \cdot)$ and its associated length scale $h\in \mathbb{R}_{>0}$. 
These parameters define the Reproducing Kernel Hilbert Space in which visitation data is stored. 
We evaluate the sensitivity of the resulting trajectory $x(t)$ to variations in the visitation embedding's length scale $h$ and the choice of kernel definition (e.g., RBF). 
We seek to characterize the qualitative balance between local visitation resolution and global regularity of the resulting control law.
Simulations are performed on a manifold represented by samples from the Stanford Dragon mesh, normalized to a compact domain $\Omega = [-0.5, 0.5]^3$, with a fixed planning horizon $N=100$ and mission duration $T=500$. 

\subsubsection{Results}
The influence of the length scale $h$ on the visitation density and qualitative trajectory behavior is given in Fig. \ref{fig:ablation_stack}. 
The length scale parameterizes the strength of the repulsive forces generated by the visitation KME. 
As $h\to0$, the visitation KME approaches a summation of Dirac measures that provide a discontinuous and highly-localized representation of past visitation.
This behavior results in myopic trajectories that frequently revisit the same local neighborhoods due to prematurely vanishing gradients outside of a small radius of $x(t)$.

As the length scale increases, the repulsive force from the visitation KME expands to facilitate broader exploration. 
However, as shown in Fig. \ref{fig:ablation_stack} (Bottom), an excessively large $h$ overestimates the effective coverage produced by $x(t)$. 
When these repulsive forces dominate the MMD's attraction to high-value regions, trajectories prematurely diverge from the manifold boundary.

We validate that the proposed method is adaptable to non-Gaussian kernel functions by comparing tuned Gaussian, Laplace, and Matern Kernels (shown in Fig. \ref{fig:kernel_ablation}). 
While each kernel possesses distinct tail behaviors and smoothness properties, the fundamental ergodic behavior of each system is preserved when their respective length scales are tuned. 
This result reinforces the assertion in \cite{gretton2012} that the MMD metric is compatible with any valid Mercer kernel, which confirms that the proposed objective yields trajectories consistent with $\mu_q$ provided the kernel bandwidth is appropriately scaled to the domain $\Omega$.




\begin{figure}[t!]
    \centering
    \vspace{7pt} 
    \includegraphics[width=\linewidth]{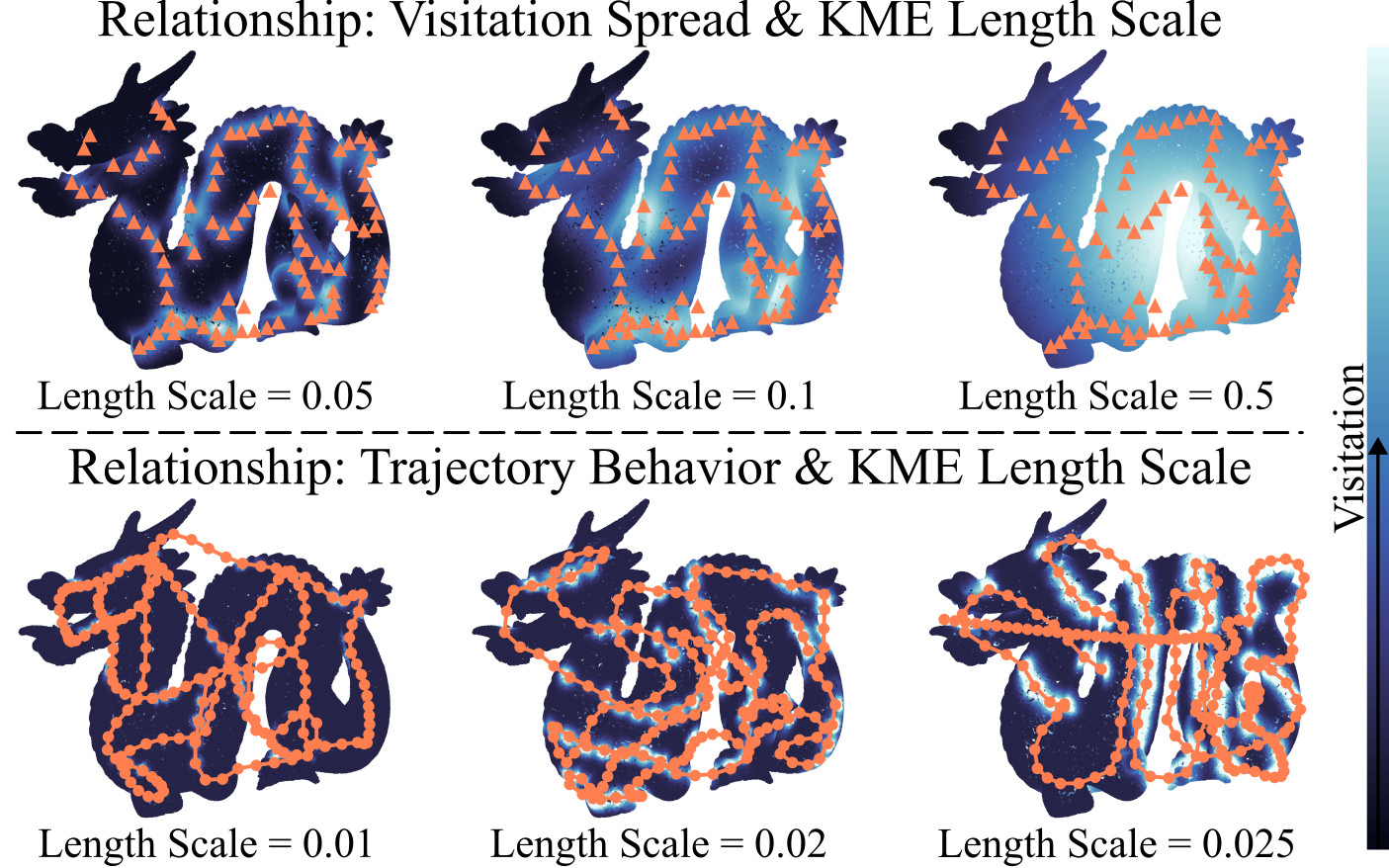}
    \vspace{-20pt}
    \caption{\textbf{Kernel Parameters on Visitation Density \& Trajectory Behavior.} 
    We evaluate the sensitivity of the kernel mean embedding (KME) and ergodic controller on the Stanford Dragon mesh with a uniform target distribution. \textbf{(Top)} Impact of KME length scale on spatial spread of visitation history. Increasing the length scale yields a coarser, more expansive representation of past coverage. \textbf{(Bottom)} Influence of length scale selection on receding-horizon trajectories. 
    As the length scale increases, so does the separation of the trajectory between itself and the mesh. 
    }
    \vspace{-10pt}
    \label{fig:ablation_stack}
\end{figure}

\begin{figure}[t!]
    \centering
    \includegraphics[width=\linewidth]{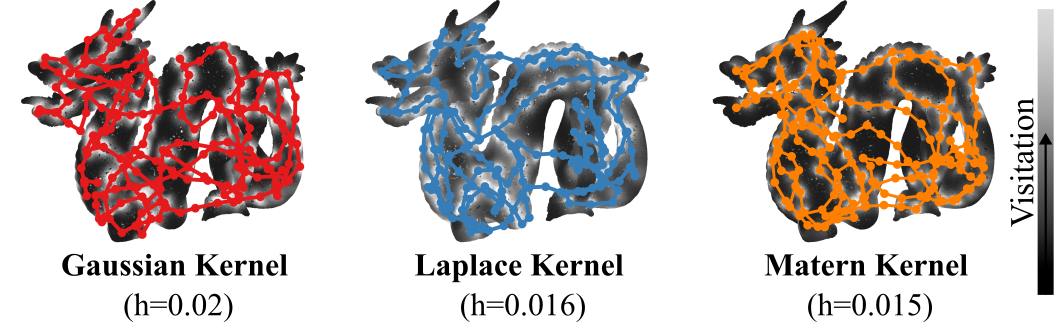}
    \vspace{-20pt}
    \caption{\textbf{Ergodic Control Trajectories Using Different Kernel Definitions.} 
    We demonstrate that our controller can adapt to alternative kernel definitions on the Stanford Dragon mesh. While each kernel possesses subtly different behaviors and smoothness properties, all three successfully drive trajectories toward uniform coverage when their respective length scales are tuned. 
    }
    \vspace{-15pt}
    \label{fig:kernel_ablation}
\end{figure}

\subsection{Computational Complexity \& Long-Horizon Scalability}
\label{sec:time_scaling}

\subsubsection{Problem Formulation}
In traditional maximum mean discrepancy formulations, the evaluation of the repulsive gradient requires an $O(t)$ summation over all past states, which results in an $O(t^2)$ cumulative complexity for a mission of duration $T$. 
We evaluate the computational scaling of the proposed KME formulation against three representative ergodic baselines:
\begin{itemize}
    \item \textbf{KME}, our method derived from kernel mean embeddings with error visitation dynamics~\eqref{eq:ivp_error_state},
    \item \textbf{Full-History}, which retains all past visited trajectory history,
    \item \textbf{Short-Term Memory}, which retains the trajectory history past for $K$ discrete-points. 
    \item \textbf{Subsampled Memory}, which uniformly samples $K$ points from the full visitation history.
\end{itemize}
Benchmarks are conducted on a manifold $\mathcal{M} \subset \mathbb{R}^3$ defined by the Stanford Bunny mesh ($M=1000$ spatial samples) over 10 independent trials. 

\subsubsection{Results}
The computational complexity relative to the discrete mission length $T\in\mathbb{Z}_+$ is shown in Fig. \ref{fig:time_scaling} (Left). 
While standard EMMD formulations theoretically exhibit quadratic growth in cumulative cost, practical implementations often suffer from super-polynomial scaling due to the increasing overhead of kernel matrix operations and memory bandwidth constraints. 

In contrast, our proposed planner demonstrates computational invariance to mission length, where the per-step computational cost scales $O(1)$ for a fixed domain and planning horizon. 
For the given mesh resolution, the crossover point, where the proposed method is computationally superior to baselines, occurs at $t\approx30$ iterations. 
Although memory-limited baselines eventually achieve constant-time scaling once they saturate their visitation memory, they necessitate an explicit loss of historical visitation. 
However, the KME-based planner preserves an implicit infinite-horizon representation of $\hat{\mu}_{\rho_{x,t}}$ without increasing the dimensionality of the optimization problem. 

Figure \ref{fig:time_scaling} (Middle and Right) confirms that the proposed method preserves traditional EMMD scaling laws with respect to the number of spatial samples $M$ and the planning horizon $N$. 
Specifically, the $O(1)$ scaling with respect to mission length does not introduce significant overhead in other planning dimensions. 
This validates that the proposed method is uniquely suited for long-horizon autonomous exploration and the respective control law remains real-time tractable over arbitrarily-large mission lengths while maintaining a consistent memory of the visitation distribution.

\subsection{Exploration Performance \& Coverage Quality}

\begin{figure}[t]
    \vspace{10pt}
    \centering
        \includegraphics[clip, trim={1em 1.25em 0em 1em}, width=0.9\linewidth]{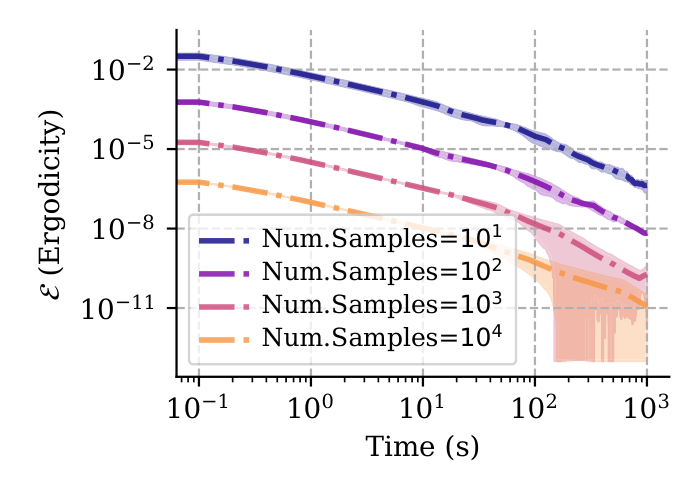}
    \vspace{-10pt}
    \caption{\textbf{Convergence to Ergodicity}. Here, we demonstrate the convergence property of the ergodic controller~\eqref{eq:ergodic_ctrl_law} derived in Theorem~\ref{thm:convergence} for the single-integrator dynamics for varying sample-size approximation of the integral. Shown is the mean and standard error ergodicity for 5 random seeds. Samples are drawn directly from the target distribution $q$ described in Fig.~\ref{fig:one_horizon}. Monotonically decreasing ergodicity and ergodic behavior induced by the controller~\eqref{eq:ergodic_ctrl_law} is observed even in low-sample settings. }
    \label{fig:convergence}
\end{figure}

\subsubsection{Problem Formulation}
Finally, we evaluate the coverage quality of the KME-based planner against several baselines, including the baselines discussed in Sec. \ref{sec:results}.C, a Traveling Salesperson (TSP) heuristic, and a Next-Best View (NBV) greedy planner. 
All solvers are evaluated on a manifold represented by the Stanford Bunny mesh $\Omega=[-0.5, 0.5]^3$ ($M=1000$) with a fixed mission duration $T=500$, planning horizon $N=150$, and a maximum control effort constraint $\lVert\Delta x_{\text{max}}\rVert=0.05$ per step. 
To ensure a rigorous comparison under equivalent computational complexity against the ergodic baselines, the K-limited baselines are constrained to $K=30$ points, matching the per-iteration overhead of the proposed method. 

\begin{figure}[t!]
    \centering
    \vspace{5pt} 
    \includegraphics[width=\linewidth]{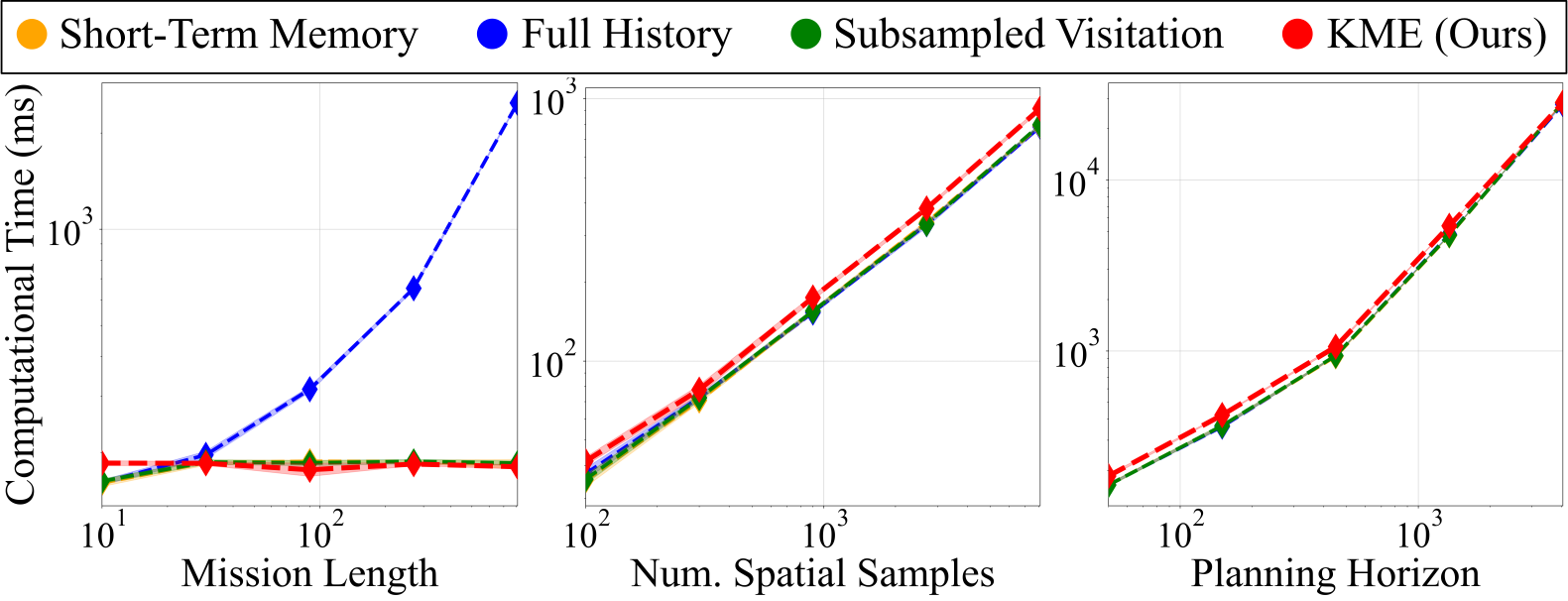}
    \vspace{-15pt}
    \caption{\textbf{Computational Scaling of Infinite-Horizon Ergodic Control.} We compare the computational scaling of the proposed approach against ergodic baselines over 10 independent trials on the Stanford bunny mesh with single-integrator dynamics. \textbf{(Left)} Recording all past visitations causes the ergodic controller to have exponential growth in compute compared to fixed-length variations. 
    \textbf{(Middle)} All methods are shown to scale linearly with the number of spatial samples $M$ needed to compute the kernel mean embedding. \textbf{(Right)} Likewise, our approach retains the computational scaling if used in a receding-horizon formulation in~\eqref{eq:plannig_opt} which only scales to the discrete planning time. 
    }
    \label{fig:time_scaling} 
    \vspace{-15pt}
\end{figure} 

\begin{figure}[b!]
    \centering
    \vspace{5pt} 
    \includegraphics[width=\linewidth]{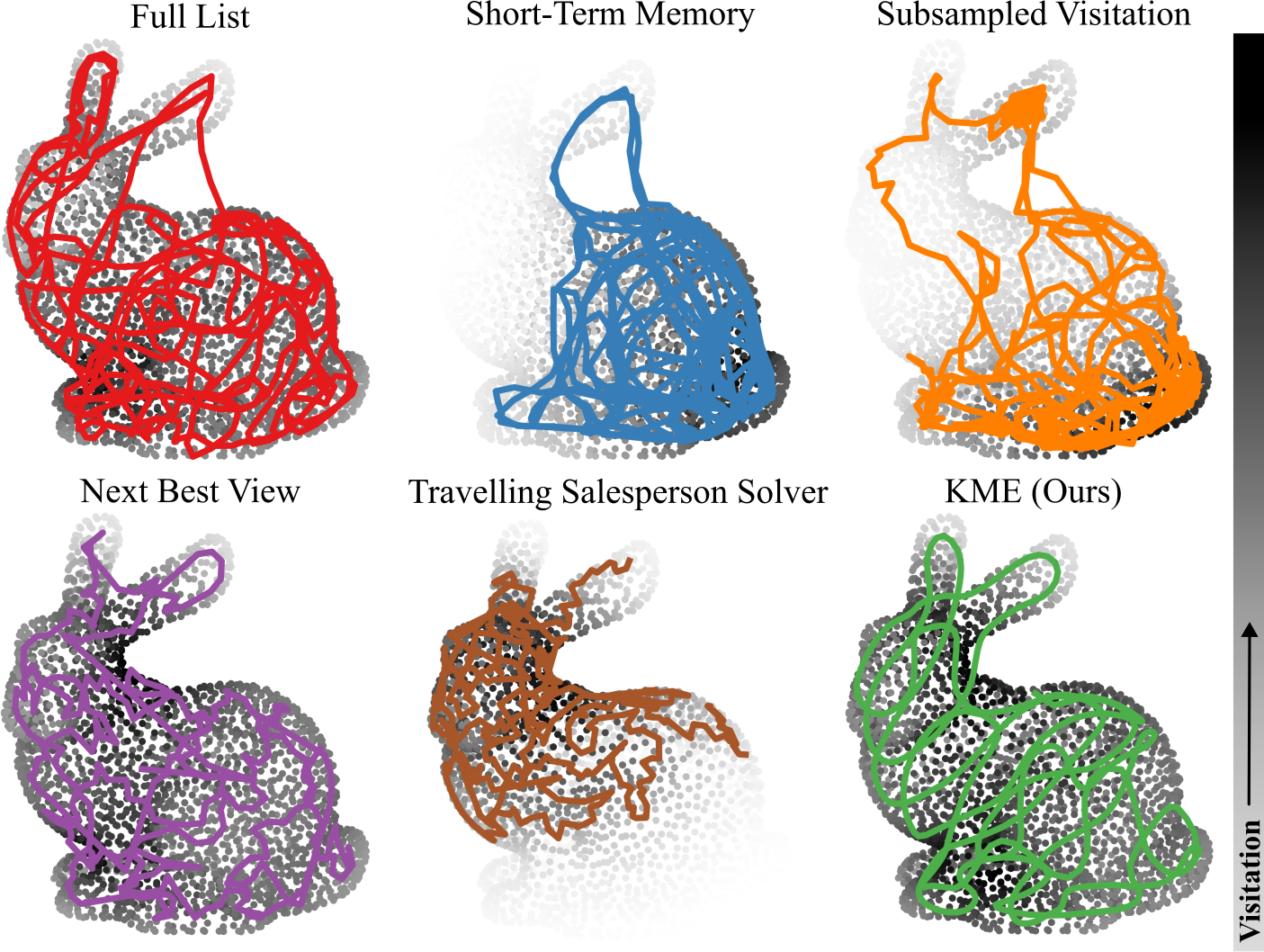}
    \vspace{-21pt}
    \caption{\textbf{Qualitative Coverage Performance Against Baselines.} Trajectory snapshots are taken at the time stamp where our KME-based method first achieves 100\% coverage of the Stanford Bunny Mesh. Our KME approach achieves coverage quality nearly identical to the computationally expensive full-memory EMMD, which validates the fidelity of the kernel mean embedding's visitation. When constrained to the same computational budget per control cycle ($K\approx30$ points), the short-term and sampled memory baselines achieve only $\approx50\%$ coverage, while the TSP solver reaches only $\approx40\%$. While a more traditional method like next-best view eventually achieves full coverage, they lack proportional coverage guarantees and fall short in trajectory smoothness.}
    \label{fig:compare_3d_qual}
\end{figure} 

\subsubsection{Results}
The qualitative and quantitative exploration results are presented in Figs. \ref{fig:compare_3d_qual} and \ref{fig:comp_quant}, respectively. 
The recursive KME-based planner achieves the highest rate of convergence to $\mu_q$. 
As shown in Fig. \ref{fig:comp_quant}, K-limited baselines exhibit significant performance degradation due to the loss of historical visitation information. 
Specifically, as the memory buffer saturates, the short-term memory approach loses global context for its exploration, which causes the control law to converge to local limit cycles. 
While the subsampled baseline provides a broader temporal window, the stochastic nature of the sampling introduces variance in the repulsive gradient, which leads to suboptimal coverage trajectories. 

In contrast, the KME-based planner maintains a high-fidelity representation of the full-horizon visitation history without the $O((t_h+t_p)^2)$ cumulative complexity of traditional EMMD. 
A key advantage of the proposed method is the ability to decouple the visitation embedding's kernel and length scale from the MMD objective's kernel. 
By prescribing a slightly larger bandwidth ($h=0.02$) for the visitation KME, we can induce a more expansive repulsive field.
This allows the recursive planner to outperform even the full-history EMMD baseline in coverage speed, as shown in Fig. \ref{fig:comp_quant}. 

Furthermore, while the NBV planner achieves competitive coverage, it lacks the formal distribution-matching guarantees inherent to ergodic control, which results in less proportionate visitation across the domain. 
The TSP-based heuristic demonstrates the lowest convergence rate, as it lacks a global incentive for uniform coverage. 
These results confirm that the proposed KME-based method provides a balance between infinite-horizon memory fidelity, computational invariance to mission length, and tunable exploration dynamics for long-horizon autonomous missions.

\begin{figure}[t]
    \centering
    \includegraphics[width=\linewidth]{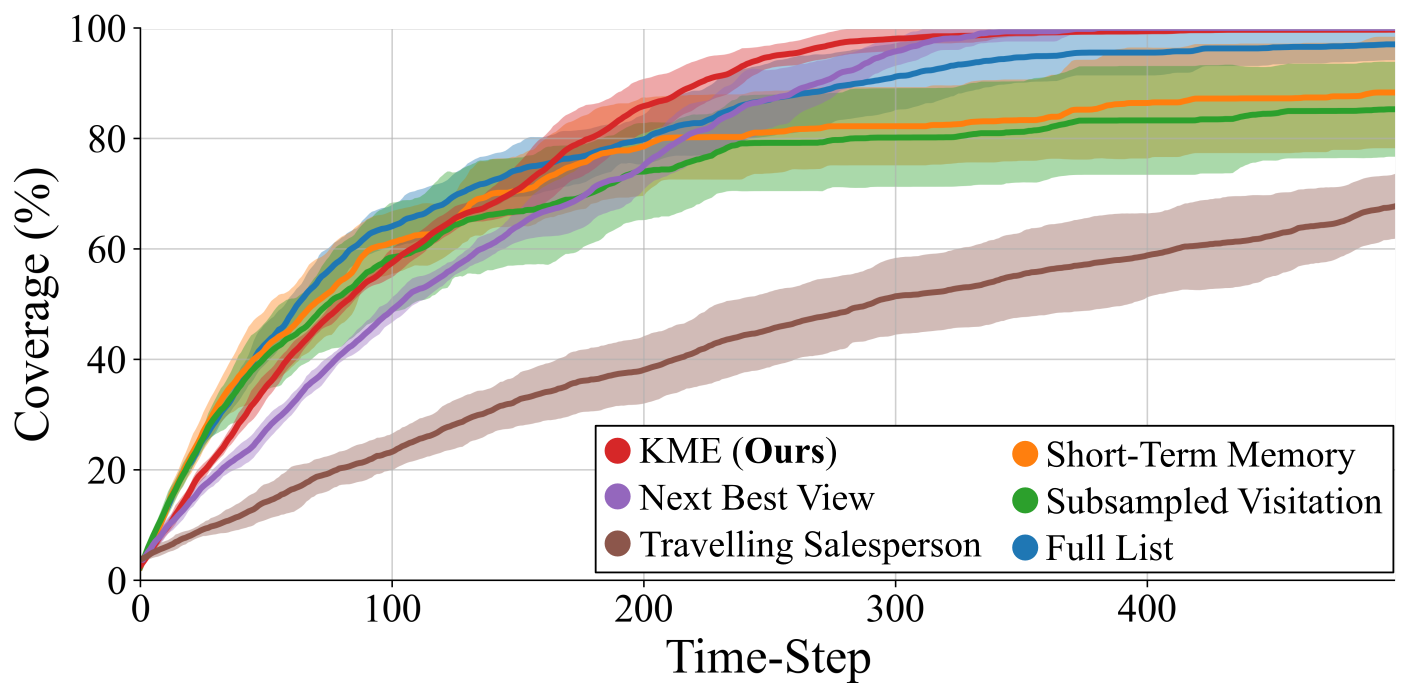}
    \vspace{-16pt}
    \caption{\textbf{Quantitative Coverage Rate.} We compare coverage performance across 10 independent trials with separate initial positions placed on random seeded points in the Stanford Bunny mesh. The proposed KME-based controller achieves the highest coverage at the fastest rate. 
    Short-term memory methods fall into limit cycles and repeatedly traverse the same paths once previous visitation data is discarded. Subsampled memory exhibits high variance due to the stochastic nature of history representation. Next Best View provides competitive coverage but is slower at full coverage. }
    \label{fig:comp_quant}
    \vspace{-10pt}
\end{figure}


\section{Conclusion}
\label{sec:conclusion}

In this paper, we present a computationally-bounded MPC objective for ergodic trajectory planning in complex, non-Euclidean domains. 
By leveraging the recursive properties of kernel mean embeddings (KMEs), we reformulate the ergodic maximum mean discrepancy (EMMD) objective to decouple an agent's visitation history from the trajectory optimization process. 
This formulation ensures that the per-step computational complexity remains temporally invariant ($O(1)$) with respect to the mission duration, allowing infinite-horizon exploration tasks without the memory overhead inherent to traditional implementations. 

We establish the theoretical foundation of this approach by deriving a proof of asymptotic convergence for the recursive objective to ensure that the proposed method preserves the proportional visitation guarantees of ergodic theory. 
Numerical validation across two- and three-dimensional manifolds demonstrates that our approach maintains high-frequency control performance and superior coverage fidelity over extended mission durations by significantly outperforming memory-intensive baselines in both computational scaling and coverage quality. 
Future work will investigate the extension of this recursive KME-based planner to time-varying domains and flow-field dynamics to facilitate robust ergodic planning in non-stationary environments. 



\bibliographystyle{IEEEtran}
\bibliography{references}



\end{document}